\documentclass[accepted]{uai2025} 
                        
\usepackage[american]{babel}

\usepackage{natbib} 
\usepackage{mathtools} 
\usepackage{booktabs} 
\usepackage{tikz} 
\usepackage{enumitem}
\usepackage{amsfonts,latexsym,amsmath,stmaryrd,dsfont}
\usepackage{multido}
\usepackage{amsthm,amssymb}
\usepackage{pst-plot}
\usepackage{graphicx}
\usepackage{fancybox}
\usepackage[normalem]{ulem}
\usepackage{amscd}
\usepackage[utf8]{inputenc}
\usepackage{caption}
\usepackage{subcaption}
\usepackage{tikz}
\usetikzlibrary{arrows.meta, positioning}
\usepackage{bbold}
\usepackage[ruled,vlined]{algorithm2e}
\usepackage{url}
\usepackage[dvips]{epsfig}
\usepackage{hyperref}
\hypersetup{colorlinks = true, linkcolor =blue, anchorcolor = red, citecolor = blue, urlcolor = blue}
\hypersetup{linktocpage}

\newtheorem{theorem}{Theorem}
\newtheorem{lemma}[theorem]{Lemma}
\newtheorem{corollary}[theorem]{Corollary}

\newtheorem{defn}[theorem]{Definition}
\newtheorem{prop}[theorem]{Proposition}

\newtheorem{obs}[theorem]{Observation}

\newcommand{\suppress}[1]{}

\def\qed{\mbox{ }~\hfill~$\Box$}

\def\cA{{\mathcal A}}

\def\cC{{\mathcal C}}

\def\E{{\mathbb E}}

\def\cH{{\mathcal{H}}}
\def\cV{{\mathcal{V}}}

\DeclareMathOperator{\Pa}{Pa}
\DeclareMathOperator{\MEC}{MEC}

\newcommand{\be}{\begin{equation}}
\newcommand{\ee}{\end{equation}}
\newcommand{\bea}{\begin{eqnarray}}
\newcommand{\eea}{\end{eqnarray}}
\newcommand{\bean}{\begin{eqnarray*}}
\newcommand{\eean}{\end{eqnarray*}}

\newcommand{\rgc}{\textsc{Reverse}(G, C)}

\newenvironment{proofof}[1]{\mbox{} \newline {\noindent \em Proof of #1.}\/}{\hfill \qed}

\usetikzlibrary{graphs}

\title{Lower Bounds on the Size of Markov Equivalence Classes}

\author[1]{\href{mailto:<ejahn@caltech.edu>?Subject=Your UAI 2025 paper}{Erik~Jahn}{}}
\author[2]{Frederick~Eberhardt}
\author[1]{Leonard~J.~Schulman}
\affil[1]{%
    Engineering and Applied Science\\
    California Institute of Technology\\
    Pasadena, California, USA
}
\affil[2]{%
    Humanities and Social Sciences\\
    California Institute of Technology\\
    Pasadena, California, USA
}
  
\begin{document}
\maketitle

\begin{abstract}
  Causal discovery algorithms typically recover causal graphs only up to their Markov equivalence classes unless additional parametric assumptions are made. The sizes of these equivalence classes reflect the limits of what can be learned about the underlying causal graph from purely observational data. Under the assumptions of acyclicity, causal sufficiency, and a uniform model prior, Markov equivalence classes are known to be small on average. In this paper, we show that this is no longer the case when any of these assumptions is relaxed. Specifically, we prove exponentially large lower bounds for the expected size of Markov equivalence classes in three settings: sparse random directed acyclic graphs, uniformly random acyclic directed mixed graphs, and uniformly random directed cyclic graphs. 
\end{abstract}

\section{Introduction}

One of the most powerful contributions of the theory of causal graphs \citep{Pea09,PC} is a complete characterization of which causal relationships can be learned from observational data without conducting experiments. This characterization is given by the concept of \emph{Markov equivalence}. Two causal models are Markov-equivalent if they encode the same conditional independence constraints. As a result, even a perfect causal discovery algorithm can only recover the true causal structure up to its Markov equivalence class, unless additional identifiable structure is assumed or happens to be present. Remarkably, Markov equivalence classes for directed acyclic graphs (DAGs) have a simple graphical characterization, due to~\citet{verma-pearl}. This is exploited by some of the most widely used causal discovery algorithms today, all of which return a Markov equivalence class as their output \citep{PC, GES, GPS}. Consequently, the size of these equivalence classes is a key measure for how informative the output of such algorithms is, making the challenge of counting the number of Markov equivalent graphs a subject of ongoing research \citep{counting_MECS, counting_MECs_2, counting_MECs_3}.

The largest Markov equivalence class of DAGs on $n$ variables has size $n!$, consisting of all the fully connected DAGs. However, for small graphs, numerical simulations \citep{Gillispie01, GILLISPIE02} and recursive enumeration \citep{STEINSKY2003, GILLISPIE2006, Steinsky2013} have shown that the average number of DAGs per Markov equivalence class is surprisingly small, even less than four. More recently, \citet{SS24} proved that for arbitrarily large $n$, the expected size of the Markov equivalence class of a uniformly random DAG on $n$ variables is bounded by a constant. These results crucially rely on the uniformity assumption, effectively placing almost all the weight on dense DAGs. In practice, however, we expect useful real-world causal structures to be sparse. Indeed, many causal discovery algorithms work best under a sparsity assumption \citep{efficient_PC, efficient-GES} and sparse priors for learning graphical models have been shown to improve practical performance \citep{sparse-learning, eggeling19a}. Numerous researchers have posed the question of how large Markov-equivalence classes are for sparse graphs \citep{Chickering-MEC-traversal, asking_our_question, interventional-MECs}.

In this paper, we give a first theoretical answer to this question. We show that for a wide range of sparse random DAG distributions, the expected size of the Markov equivalence classes basically scales exponentially in the inverse edge density. In particular, when the expected degree of each vertex is bounded by a constant, the Markov equivalence classes are exponentially large in the number of vertices of the DAG in expectation. This result reveals a sharp contrast to the uniform setting and also has algorithmic implications: Many greedy search algorithms can be run either in the space of DAGs or in the space of equivalence classes, but with small equivalence class sizes the efficiency gain remained unclear \citep{Chickering-MEC-traversal}. Our results suggest that the efficiency gap becomes larger as the input graphs are sparser.

Moreover, we extend our analysis of Markov equivalence classes beyond DAGs. Causal models represented by DAGs inherently make two assumptions: \emph{causal sufficiency}, meaning that there are no unobserved common causes, and \emph{acyclicity}, implying the absence of causal feedback loops. Both assumptions may be violated in applications, motivating the study of more general graphical models: acyclic directed mixed graphs (ADMGs), which allow for unmeasured confounders, and directed cyclic graphs (DCGs), which permit cyclic causal relationships. Just as for DAGs, Markov equivalence has been characterized for both models \citep{MECs_in_ADMGs, characterizing_MECs_in_DCGs} and corresponding causal discovery algorithms that output Markov equivalence classes have been developed \citep{FCI, richardson2018discovering}. The characterization of Markov equivalence classes is much more complex in both cases, and little is known about the sizes of these Markov equivalence classes. Here, we establish exponential lower bounds for the expected size of Markov equivalence classes of uniformly random ADMGs and DCGs. Our results exploit specific underdetermined substructures, leaving open the possibility that more restrictive model classes, such as maximal ancestral graphs \citep{mags}, may exhibit smaller Markov equivalence classes by prohibiting such structures. 
In general, our results highlight the need to supplement causal discovery with stronger, possibly parametric assumptions, or to perform interventions if one wants to reduce the size of the set of plausible graphs, let alone find a unique causal graph that explains the data. 

\subsection{Our Results}
In this section, we provide formal statements of our results. For a parameter $p \in [0,1)$, consider the following natural sampling process for generating a random DAG on $n$ vertices: 
\begin{itemize}
\item Include each possible directed edge independently with probability $p$;
\item If the graph has any directed cycles, reject and repeat. \end{itemize}
We denote the distribution arising from this process by $D(n,p)$ (see also Definition~\ref{defn:Dnp} for an alternative description). The parameter $p$ directly controls the edge density of the sampled graph, and one can check that $D(n, 1/2)$ is equal to the uniform distribution on vertex-labeled DAGs. \citet{heckerman1995learning} already used this distribution in the context of Bayesian networks and \citet{eggeling19a} showed that it has desirable properties as a prior for Bayesian inference over causal structures. While the sampling process described above becomes impractical for large $n$ due to a high rejection probability, a more efficient sampling algorithm for $D(n,p)$ has been given by \citet{talvitie}.

We study the sizes of Markov equivalence classes of DAGs sampled from $D(n,p)$ in the regime $6/n \leq p \leq o(1/\log n)$. The graphs in this regime have expected vertex degree ranging from constant (whenever $p = C/n$) up to $o(n/\log n)$. It turns out that the number of graphs that are Markov-equivalent to a DAG $G$ sampled from $D(n,p)$ scales at least almost exponentially in the inverse of $p$. 

\begin{theorem}\label{thm:DAGs}
    Let $6/n \leq p \leq o\left(\frac{1}{\log n}\right)$ and $G \sim D(n,p)$. Then, we have with probability $1-o(1)$: 
    \begin{equation*}
        |\MEC(G)| \geq 2^{\Omega\left(\frac{p^{-1}}{\log^2(p^{-1})}\right)}.
    \end{equation*}
    In particular, this implies:
    \begin{equation*}
        \E[|\MEC(G)|] \geq 2^{\Omega\left(\frac{p^{-1}}{\log^2(p^{-1})}\right)}.
    \end{equation*}
\end{theorem}

For the more general settings of ADMGs and DCGs, we show that the expected size of the Markov equivalence class is (super-)exponential in the number of vertices when the graph is sampled uniformly at random. 

\begin{theorem}\label{thm:MECs}
    Let $G$ be a uniformly random ADMG on $n$ vertices, and let $H$ be a uniformly random DCG on $n$ vertices. Then, we have 
    \begin{enumerate}[label=(\alph*)]
        \item $\E[|\MEC(G)|] \geq 2^{\Omega(n^2)}$; \label{item:admg}
        \item $\E[|\MEC(H)|] \geq 2^{\Omega(n)}.$ \label{item:dcg}
    \end{enumerate}
\end{theorem}

In fact, our proof for Theorem~\ref{thm:MECs} shows that each ADMG has on average already exponentially many Markov equivalent graphs that all have different edge adjacencies. This stands in contrast to the DAG setting where all Markov equivalent DAGs have the same adjacencies \citep{verma-pearl}. For DCGs, we show that each graph has on average already exponentially many Markov equivalent graph that differ in the direction of their cycles.

The rest of this paper is structured as follows: in Section~\ref{sec:preliminaries}, we provide definitions and notation for graphical models, and state some useful known results. In section~\ref{sec:DAGs} we present an outline of our approach towards proving Theorem~\ref{thm:DAGs} for DAGs (with full technical proofs in the appendix). The proof of Theorem~\ref{thm:MECs}, part~\ref{item:admg} for ADMGs can be found in Section~\ref{sec:ADMGs} and the proof of Theorem~\ref{thm:MECs}, part~\ref{item:dcg} for DCGs is given in Section~\ref{sec:DCGs}. We discuss the implication of our results and further open questions in Section~\ref{sec:Discussion}.

\section{Preliminaries}\label{sec:preliminaries}

\subsection{Graphical Models}
A \emph{directed graph} $G$ consists of a \emph{vertex set} $V(G)$ and a set of \emph{directed edges} $E(G)$, which are ordered pairs $(v,w)$ of vertices $v \neq w \in V(G)$. A \emph{directed mixed graph} additionally has a set of bidirected edges that are unordered pairs of vertices $\{v,w\}$ for $v \neq w$. We denote directed edges by $v \to w$ and bidirected edges by $v \leftrightarrow w$. For a subset of vertices $W \subseteq V(G)$, we define the \emph{induced} graph $G[W]$ as the graph whose vertex set is $W$ and whose edges are the edges of $G$ that lie entirely within $W$. A (possibly self-intersecting) \emph{path} in a directed mixed graph $G$ is an ordered list of vertices $(v_1, \dots, v_k)$ with $v_1 \neq v_k$ such that there is a directed edge (in either direction) or bidirected edge between any two consecutive vertices in the list. A \emph{directed path} is an ordered list of vertices $(v_1, \dots, v_k)$ with $v_1 \neq v_k$ such that there is a directed edge from $v_i$ to $v_{i+1}$ for $i=1, \dots, k-1$. A \emph{cycle} $C = (v_1, \dots, v_k)$ is a directed path that additionally has a directed edge from $v_k$ to $v_1$ (all cycles are directed). In the context of a cycle of length $k$, we usually consider all indices modulo $k$, in particular, we identify $v_{k+1} = v_1$. The \emph{parents} of a vertex $v$ are the vertices $u$ that have a directed edge $u \to v$, the \emph{children} of $v$ are the vertices $w$ that have a directed edge $v \to w$, and the \emph{descendants} of $v$ are the vertices $x$ such that there is a directed path from $v$ to $x$. We denote the set of parents of $v$ by $\Pa(v)$. A \emph{source} of the graph $G$ is a vertex without parents. A \emph{matching} of edges is a set of edges that do not have any vertices in common (i.e.\ their sets of endpoints are pairwise disjoint from each other). A \emph{directed acyclic graph (DAG)} is a directed graph that contains no cycles and a \emph{directed cyclic graph (DCG)} is a directed graph that may or may not contain cycles (this terminology has been used in the literature to clearly distinguish from the more popular DAGs in the context of causal models). An \emph{acyclic directed mixed graph (ADMG)} is a directed mixed graph that contains no (directed) cycles. We study random DAGs under the following probability distribution:

\begin{defn} \label{defn:Dnp}
    Fix a positive integer $n$ and a parameter $p \in [0,1)$. We define the distribution $D(n,p)$ over DAGs on $n$ vertices as the distribution that assigns each DAG $G$ with $e$ edges a probability proportional to $w(G) = \left(\frac{p}{1-p}\right)^e$.
\end{defn}

Equivalently, $D(n,p)$ is the distribution that arises from sampling each possible directed edge independently with probability $p$ (resulting in $\Pr(G) = p^e(1-p)^{n(n-1) - e}$) and then conditioning on acyclicity. The distribution $D(n,1/2)$ is equal to the uniform distribution over DAGs on $n$ vertices. Note that random DAGs can also be obtained by sampling edges from an upper triangular matrix and uniformly permuting vertex labels. However, this process places a bias on DAGs with many automorphisms and cannot be seen as a natural extension of the uniform distribution. For  ADMGs, a uniformly random ADMG is obtained by sampling a uniformly random DAG, and then adding a bidirected edge for each pair of vertices independently with probability $1/2$. A uniformly random DCG is obtained by placing each possible directed edge independently with probability $1/2$.

\subsection{Markov Equivalence}
The following definitions hold for all three model classes (DAGs, ADMGs, DCGs) alike. Given a path $\pi = (v_1, \dots, v_k)$, the vertex $v_i$ is a \emph{collider} on the path if there are two incoming arrows from $v_{i-1}$ and $v_{i+1}$ to $v_i$, that is, one of the following holds true: $v_{i-1} \to v_i \leftarrow v_{i+1}$, $v_{i-1} \to v_i \leftrightarrow v_{i+1}$, $v_{i-1} \leftrightarrow v_i \leftarrow v_{i+1}$, $v_{i-1} \leftrightarrow v_i \leftrightarrow v_{i+1}$. Given two vertices $v,w \in V(G)$ and a \emph{conditioning set} $Z \subseteq V(G) \setminus \{v,w\}$, a path $\pi$ from $v$ to $w$ is \emph{active given $Z$} if every non-collider vertex on the path is not in $Z$ and every collider vertex on the path is either in $Z$ or has a descendant in $Z$. If there is an active path between $v,w$ given $Z$, then $v$ and $w$ are said to be d-connected given $Z$; otherwise they are d-separated given $Z$. Two graphs $G_1$ and $G_2$ on the same vertex set $V$ are \emph{Markov-equivalent} if for all $v,w \in V$ and $Z \subseteq V \setminus \{v,w\}$, $v,w$ are d-connected given $Z$ in $G_1$ if and only if they are d-connected given $Z$ in $G_2$. The set of all graphs that are Markov-equivalent to a graph $G$ is called the \emph{Markov equivalence class (MEC)} of $G$. The significance of Markov equivalence in causal inference stems from the fact that two causal models represented by Markov-equivalent graphs encode the same conditional independence structure. The connection between this property and the graphical definition of Markov-equivalence through d-separation has first been formalized for DAGs by \citet{d-separation, Geiger90}. Later, it has been shown that the same connection holds for ADMGs \citep{Spirtes_ADMGs, Koster_ADMGs} and DCGs   \citep{DCGs_d-separation} (at least for linear models). Hence, for these models, Markov-equivalent graphs cannot be distinguished based on observing conditional dependence and independence relations.

For DAGs, we will additionally make use of the following result: Let us call an edge $v \to w$ of a DAG $G$ \emph{reversible} if replacing the edge by $w \to v$ results in another DAG $G'$ that is Markov equivalent to $G$. Then, reversible edges have a simple characterization: 

\begin{lemma}{\citep{chickering95}}\label{lem:reversible}
    An edge $v \to w$ of a DAG $G$ is reversible if and only if $\Pa(v) = \Pa(w) \setminus \{v\}$.
\end{lemma}

\subsection{Tower Decomposition of DAGs}
\begin{figure}
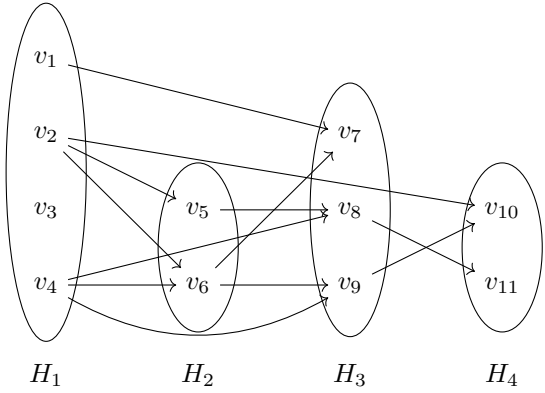

    \centering
    \include{tower}
    \caption{A DAG on $11$ vertices with tower decomposition $(H_1, H_2, H_3, H_4)$.}
    \label{fig:tower}
\end{figure}

The definitions and notation in this subsection are from \citet{SS24}. We define the \emph{tower decomposition} of a directed acyclic graph $G$ as follows: Let $H_1(G)$ be the set of sources in $G$, $H_2(G)$ the set of sources in $G \setminus H_1(G)$, $H_3(G)$ the set of sources in $G \setminus (H_1(G) \cup H_2(G))$, and so on. This partitions the vertex set $V(G)$ into sets $H_1, H_2 \dots, H_{s(G)}$, where $s(G)$ is the length of the longest directed path in $G$ (we suppress the dependence on $G$ if it is clear from the context). See Figure~\ref{fig:tower} for an example. We call the sets $H_i$ the \emph{layers} of $G$ and denote their sizes by $h_i := |H_i|$ for $i=1, \dots, s$. The vector $h(G) = (h_1, \dots, h_{s})$ is called the \emph{tower vector} of $G$. The tower decomposition $(H_1, \dots, H_s)$ of $G$ is characterized by two properties: first, edges of $G$ can only go from a layer $H_i$ to a layer $H_j$ with $j>i$, in particular, there are no edges within the layers themselves. Secondly, each vertex in $H_i$ must have at least one parent in $H_{i-1}$ (except the vertices in $H_1$, which are exactly the vertices in $G$ that have no parents).We use tower decompositions to obtain a useful description of the distribution $D(n,p)$: 

\begin{lemma}\label{lem:tower-distr}
    Let $G \sim D(n,p)$, and fix a tower vector $h = (h_1, \dots, h_s)$ with entries that sum up to $n$, and a tower decomposition $H = (H_1 \dots, H_s)$ with $|H_i| = h_i$. Then, the following holds:
    
    \begin{enumerate}[label=(\alph*)]
        \item the probability that $h$ is the tower vector of $G$ is proportional to 
    \begin{align*}
        w(h) = \binom{n}{h_1, \dots, h_s} \prod_{k=2}^s\frac{\left(1 - (1-p)^{h_{k-1}}\right)^{h_k}}{(1-p)^{h_k \sum_{i=1}^{k-1}h_i}}; 
    \end{align*} \label{item:a}
        \item conditional on the event that $H$ is the tower decomposition of $G$, the parent sets $\Pa(v)$ are independently distributed for all $v \in V(G)$, and for $v \in H_j$ and $S \subseteq \bigcup_{i=1}^{j-1}H_i$, we have 
        \begin{align*}
            \Pr\left(\Pa(v) = S\right) = \begin{cases}
                C \cdot \left(\frac{p}{1-p}\right)^{|S|} \text{ if } S \cap H_{j-1} \neq \emptyset\\
                0 \text{ otherwise,}
            \end{cases}
        \end{align*}
        where $C$ is the normalization constant. \label{item:b}
    \end{enumerate}
\end{lemma}

Here, $\binom{n}{h_1, \dots, h_s}$ denotes the multinomial coefficient. Lemma~\ref{lem:tower-distr} generalizes Lemma 3.1 of \citet{SS24} and essentially follows as a special case of Lemma 3 and equation (2) of \citet{talvitie}. Still, for completeness, we give a full proof of Lemma~\ref{lem:tower-distr} in Appendix~\ref{apx:dags}. Note that tower decompositions have also been studied under different names in the literature, such as root layerings \citep{talvitie} and DAG partitions \citep{Kuipers2015, Kuipers02012017}. 

For uniformly random DAGs, we will additionally make use of some results of \citet{SS24}. Let $\sigma(G)$ be the graph obtained from $G$ by keeping only the edges between adjacent layers $H_i$ and $H_{i+1}$ for all $i$. We call this graph the \emph{tower} of $G$. Conditioning on the tower of $G$ drastically simplifies the distribution of a uniformly random DAG: 

\begin{lemma}{\citep{SS24}}\label{lem:tower_conditioning}
    Let $G$ be a uniformly random DAG on $n$ vertices and condition on the tower $\sigma(G) = \sigma$. Then, the edges of $G$ between two non-adjacent layers occur independently with probability $1/2$ each. 
\end{lemma}

The following lemma gives a tail bound on the layer sizes in a uniformly random DAG $G$. 

\begin{lemma}{\citep{SS24}}\label{lem:tower_size}
    Let $G$ be a uniformly random DAG on $n$ vertices. Let $i, \ell \in \mathbb{N}$, $\ell \geq 5$, and $n$ sufficiently large. Then, we have $\Pr[h_i(G) \geq \ell] \leq 2^{-\ell^2/4 + 2}$.
\end{lemma}

\section{Sparse Random DAGs}\label{sec:DAGs}

Our approach towards bounding the size of the Markov equivalence class of a sparse random DAG $G$ is through bounding the number of reversible edges of $G$. By Lemma~\ref{lem:reversible} an edge $v \to w$ is reversible if and only if the parent sets of $v$ and $w$ align. If we think about $G$ in terms of its tower decomposition, vertices that lie in higher-order layers have many more potential parents, so an edge between such vertices should be less likely to be reversible (in fact, \citet{SS24} derive upper bounds on the size of MECs in uniformly random DAGs by making this intuition formal). However, if we consider an edge $v \to w$ where $w$ lies in the second layer $H_2(G)$ (i.e. the set of vertices that are children of a source vertex), then $v$ is necessarily a source vertex, so $\Pa(v) = \emptyset$. But this implies that $v \to w$ is reversible whenever $w$ has no parents except $v$. Since $w$ lies in the second layer of $G$, it can only have parents in the first layer, so checking that $w$ has no connection to another source vertex suffices to establish reversibility of $v \to w$. In the following, we call an edge $v \to w$ with $w \in H_2(G)$ a \emph{layer-2-edge}. For a uniformly random DAG, the first two layers contain with high probability only a few vertices, however, this is not the case for a sparse random DAG. In fact, we will show that the first two layers of a sparse random DAG $G$ are with high probability quite large, which leads to the existence of many reversible layer-2-edges. This by itself does not yet suffice to give a significant lower bound on the size of the MEC of $G$. Indeed, two reversible edges need not be independently reversible, that is, after reversing one edge, the other one might become irreversible. However, suppose we have an edge set $S \subseteq E(G)$, which forms a \emph{matching} of reversible edges, i.e.\ a set of reversible edges with disjoint endpoints. Then, reversing an edge in $S$ only changes the parent sets of its endpoints and leaves all other parent sets the same. But this implies that all other edges in $S$ must remain reversible by Lemma~\ref{lem:reversible}. Hence, all possible combinations of edge reversals for edges in $S$ lead to another Markov equivalent DAG. We formulate this as the following observation: 

\begin{obs}\label{obs:key}
    Let $G$ be a DAG, and let $S \subseteq E(G)$ be a matching of reversible edges in $G$. Then, the size of the Markov equivalence class of $G$ is at least $2^{|S|}$.
\end{obs}

Now, the key insight of this section is that we can find a large matching of reversible edges in a sparse random graph $G$ by only looking at its layer-2-edges. 

\begin{prop}\label{prop:DAGs}
    Let $6/n \leq p \leq o\left(\frac{1}{\log n}\right)$ and $G \sim D(n,p)$. Then, with probability $1-o(1)$, $G$ has a matching of reversible layer-2-edges of size at least $\frac{p^{-1}}{16e^5\log^2(p^{-1})}$.
\end{prop}

\begin{figure}
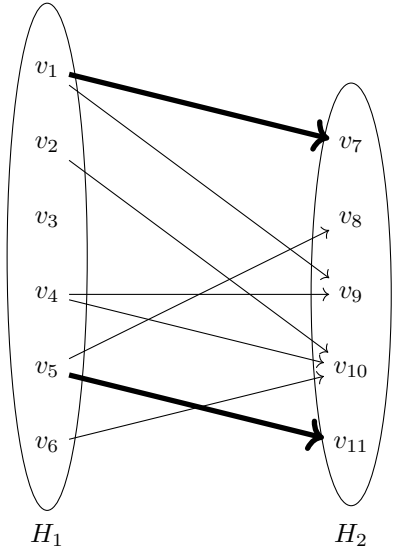

    \centering
    \include{matching}
    \caption{The bold edges form a (maximal, but not unique) matching of reversible layer-2-edges.}
    \label{fig:matching}
\end{figure}

See Figure~\ref{fig:matching} for an example of a matching of reversible layer-2-edges. To prove Proposition~\ref{prop:DAGs}, we crucially rely on the following concentration bounds for the sizes of the first two layers of a sparse random DAG, which are of independent interest:

\begin{lemma} \label{lem:upper-tailbd}
    Let $p = o\left(\frac{1}{\log n}\right)$ and $G \sim D(n,p)$. Then, the number of sources in $G$ is less than or equal to $\frac{5}{p}$ with high probability. 
\end{lemma}

\begin{lemma} \label{lem:lower-tailbd}
    Let $\frac{6}{n} \leq p \leq o\left(\frac{1}{\log n}\right)$ and $G \sim D(n,p)$. Then, with high probability, we have $h_1(G) \geq \frac{p^{-1}}{20 \log(p^{-1})}$ and $h_2(G) \geq \frac{p^{-1}}{\log^2(p^{-1})}$.
\end{lemma}
Lemma~\ref{lem:upper-tailbd} follows by showing that a tower vector $h$ with $h_1 > 5p^{-1}$ has exponentially small weight $w(h)$ compared to the vector that is obtained from $h$ by splitting $h_1$ into equal parts of size $p^{-1}$. Here, $w(h)$ is given by  Lemma~\ref{lem:tower-distr}, part~\ref{item:a}. Similarly, Lemma~\ref{lem:lower-tailbd} follows from showing that a vector $h$ with $h_1 < p^{-1}/(20 \log(p^{-1}))$ has exponentially small weight compared to the vector that is obtained from $h$ by merging its first layers until their size exceeds $p^{-1}/5$ (and similarly for $h_2$). We provide the full details of this calculation in Appendix~\ref{apx:dags}. In light of these concentration bounds, the statement of Proposition~\ref{prop:DAGs} becomes more approachable: Given a DAG $G$ sampled from $D(n,p)$, we simply need to show that there is at least a small constant fraction of vertices in $H_2(G)$ that have only one parent in $H_1(G)$. As these parents should roughly be scattered randomly with just a few overlaps, most of these vertices together with their single parent then form a matching of reversible edges. But $H_2(G)$ contains $p^{-1}/\log^2(p^{-1})$ vertices with high probability, so the matching we found has the desired size. We give a full proof of Proposition~\ref{prop:DAGs} in Appendix~\ref{apx:dags}, which, together with Observation~\ref{obs:key}, immediately implies Theorem~\ref{thm:DAGs}.

\section{Acyclic Directed Mixed Graphs}\label{sec:ADMGs}
For DAGs it is known that any two Markov equivalent graphs must have the same edge adjacencies. This is not true anymore for ADMGs. Let us call an edge $v \to w$ of an ADMG $G$ \emph{underdetermined} if deleting it results in an ADMG that is Markov equivalent to $G$. Our approach towards proving Theorem~\ref{thm:MECs}, part~\ref{item:admg} is to show that a uniformly random ADMG $G$ contains many underdetermined edges that can be deleted or included independently of each other while preserving Markov equivalence. The following graph provides an example for an underdetermined edge: 

\begin{defn}
    We define the ADMG $S$ on an ordered set of three vertices $(v_1, v_2, v_3)$ as the graph with edges $v_1 \to v_2$, $v_2 \to v_3$, and $v_2 \leftrightarrow v_3$. Moreover, we denote the graph $S \cup \{v_1 \to v_3\}$ as $\overline{S}$, see Figure~\ref{fig:admg}.
\end{defn}

\begin{figure}
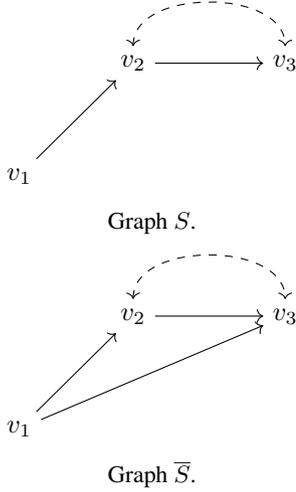

    \captionsetup[subfigure]{labelformat=empty}
    \begin{subfigure}{\linewidth}
    \centering
    \tikz{
    \node (v1) at (0,0) {$v_1$};
    \node (v2) at (1.5,1.5) {$v_2$};
    \node (v3) at (3.5,1.5) {$v_3$};

    \draw[->] (v1) to (v2);
    \draw[->] (v2) to (v3);
    \draw[dashed, <->] (v2) to[out=90,in=90] (v3);
    }
    \caption{Graph $S$.}
    \end{subfigure}
    \begin{subfigure}{\linewidth}
    \centering
    \tikz{
    \node (v1) at (0,0) {$v_1$};
    \node (v2) at (1.5,1.5) {$v_2$};
    \node (v3) at (3.5,1.5) {$v_3$};

    \draw[->] (v1) to (v2);
    \draw[->] (v2) to (v3);
    \draw[dashed, <->] (v2) to[out=90,in=90] (v3);
    \draw[->] (v1) to (v3);
    }
    \caption{Graph $\overline{S}$.}
    \end{subfigure}
    \caption{The edge from $v_1$ to $v_3$ is underdetermined: it does not introduce any new d-connections, so $S$ and $\overline{S}$ are Markov equivalent.}
    \label{fig:admg}
\end{figure}

It is easy to check that the edge $v_1 \to v_3$ is underdetermined in $\overline{S}$. A key observation is that this remains true even when $\overline{S}$ is part of a larger graph structure. 

\begin{lemma}\label{lem:admg-underdetermined}
    Let $G$ be an ADMG that contains $\overline{S}$ as a (labeled) subgraph on the vertices $(v_1, v_2, v_3)$. Then, the edge $v_1 \to v_3$ is underdetermined in $G$. 
\end{lemma}

\begin{proof}
    Let $G$ be an ADMG containing $\overline{S}$ as a labeled subgraph on $(v_1, v_2, v_3)$ and let $G'$ be the graph obtained from deleting the edge $v_1 \to v_3$ from $G$. Since $G'$ is a subgraph of $G$, any two vertices $a, b$ that are d-separated given $Z$ in $G$ are also d-separated given $Z$ in $G'$. Now, assume $a$ and $b$ are d-connected in $G$ given $Z$. First, note that each vertex has the same descendants in $G$ and $G'$, since $G'$ still contains a directed path from $v_1$ to $v_3$. Hence, if $\pi$ is an active path from $a$ to $b$ given $Z$ in $G$ that does not contain the edge $v_1 \to v_3$, then $\pi$ is also an active path from $a$ to $b$ given $Z$ in $G'$. Otherwise, if $\pi$ is an active path from $a$ to $b$ given $Z$ in $G$ that contains the edge $v_1 \to v_3$, then we distinguish two cases: \\
    First, if $v_2 \notin Z$, then replacing the edge $v_1 \to v_3$ in $\pi$ by the path $v_1 \to v_2 \to v_3$ gives an active path in $G'$. Secondly, if $v_2 \in Z$, then replacing the edge $v_1 \to v_3$ in $\pi$ by the path $v_1 \to v_2 \leftrightarrow v_3$ gives an active path in $G'$. In all cases, we get that $a$ and $b$ are also $d$-connected given $Z$ in $G'$, so $G$ and $G'$ are Markov equivalent. 
\end{proof}

This leads to the following corollary: 
\begin{corollary} \label{cor:admg-underdetermined}
    Let $G$ be an ADMG on $n$ vertices and suppose there are $m$ different subsets of vertices $V_1, \dots, V_m$ such that $|V_i| = 3$ and the induced graph $G[V_i]$ contains a copy of $S$ or of $\overline{S}$ for all $i$. Then, $|\MEC(G)| \geq 2^{m/(3n)}$.
\end{corollary} 

\begin{proof}
Given the setup as in the Corollary, note that each subset $V_i$ only intersects with at most $3(n-3)$ other subsets in two vertices. Hence, we can select at least $m/(3n)$ different subsets $V_i$ that pairwise intersect in at most one vertex. However, this implies that the copies of $S$ and $\overline{S}$ on these vertex sets are all edge-disjoint. In particular, deleting or adding the underdetermined edge in these copies does not influence the other copies, so we can add or delete an edge for each of the sub-selected vertex sets independently while preserving Markov-equivalence. This implies $|\MEC(G)| \geq 2^{m/(3n)}$.
\end{proof}

All that is left to do in order to prove Theorem~\ref{thm:MECs}, part~\ref{item:admg} is to show that in a uniformly random ADMG $G$, a constant fraction of all possible vertex sets of size $3$ is expected to have a copy of $S$ or $\overline{S}$.

\begin{proofof}{Theorem~\ref{thm:MECs}, part~\ref{item:admg}}
    Let $G$ be a uniformly random ADMG on $n$ vertices. Note that the distribution of $G$ can be described by sampling a uniformly random DAG and adding a bidirected edge for each pair of vertices independently with probability $1/2$. We extend the notion of tower decompositions to ADMGs and denote by $\sigma(G)$ the tower of the DAG formed by the directed edges of $G$. Let $\cA$ be the set of towers $\sigma$ that only have layers of size at most $n/48$. For each $\sigma \in \cA$, let $\cV_\sigma$ be the set of all vertex subsets of size $3$ that do not contain two vertices in adjacent layers. We have $|\cV_\sigma| \geq \frac{1}{6} \cdot n \cdot \frac{23}{24}n \cdot \frac{11}{12}n \geq n^3/7.$ By Lemma~\ref{lem:tower_conditioning}, after conditioning on $\sigma(G) = \sigma$, the edges induced  by each $W \in \cV_\sigma$ occur independently with probability $1/2$ each, so we get $\Pr(W \text{ contains a copy of $S$ or $\overline{S}$} \mid \sigma(G) = \sigma) \geq 1/8$. Then, if $X$ denotes the number of vertex sets $W \in \cV_\sigma$ that contain a copy of $S$ or $\overline{S}$, we get $\E[X \mid \sigma(G) = \sigma] \geq n^3/56$. By Corollary~\ref{cor:admg-underdetermined} and Jensen's inequality, we deduce
    \begin{align*}
    \E\left[|\MEC(G)| \middle | \sigma(G) = \sigma\right] &\geq \E\left[2^{X/(3n)} \middle | \sigma(G) = \sigma \right]\\ &\geq 2^{n^2/168}.
    \end{align*}

    After summing over all $\sigma \in \cA$, we conclude using Lemma~\ref{lem:tower_size}: 
    \begin{align*}
        \E\left[|\MEC(G)|\right] &\geq \Pr\left(\sigma(G) \in \cA\right) \cdot 2^{n^2/168}\\
        &= \Pr\left( h_i \leq n/48 \text{ for all $i$}\right) \cdot 2^{n^2/168}\\
        &\geq (1 - n \cdot 2^{-n^2/9216+2}) \cdot 2^{n^2/168}\\ 
        &= 2^{\Omega(n^2)}.
    \end{align*}
\end{proofof}

We remark that it is straightforward to show tight concentration of the random variable $X$, counting the occurrence of underdetermined edges in $G$ by Chebychev's inequality. This means that in fact almost all ADMGs on $n$ vertices have at least $2^{\Omega(n^2)}$ Markov-equivalent graphs.

\section{Directed Cyclic Graphs}\label{sec:DCGs}

Already \citet{richardson-cycles} noticed that the direction of cycles in a DCG is underdetermined within its Markov equivalence class. We extend this observation by giving a formal construction for how to reverse any cycle in a DCG while preserving Markov equivalence. 

\begin{defn} \label{defn:reverse}
    Let $G$ be a DCG with a cycle $C = (v_1, \dots, v_k)$. We construct the graph $H = \textsc{Reverse}(G, C)$ as follows:
    \begin{itemize}
        \item Take a copy of $G$ and reverse the cycle $C$, i.e. replace the edge $v_i \to v_{i+1}$ by $v_{i} \leftarrow v_{i+1}$ for all $i \in [k]$;
        \item For each vertex $v_i \in C$, and each vertex $w \in \Pa(v_i) \setminus \{v_{i-1}\}$, delete the edge $w \to v_i$, and replace it by $w \to v_{i-1}$.
    \end{itemize}
\end{defn}

We state the key property of this construction:  

\begin{prop}\label{prop:ME-reverse}
    Let $G$ be a directed cyclic graph containing a cycle $C$. Then, the graph $H = \textsc{Reverse}(G, C)$ is Markov equivalent to $G$. 
\end{prop}

This result and the entire $\textsc{Reverse}$ construction is based on the following intuition: consider a segment $v_{i-1} \to v_i \to v_{i+1}$ of a cycle $C$ in $G$ together with an incoming edge $w \to v_i$, see Figure~\ref{fig:constr}. If we only reverse the orientation of $C$ without changing the edge $w \to v_i$, then $w$ becomes d-separated from $v_{i-1}$ given $\{v_i, v_{i+1}\}$ in the resulting graph, while it was d-connected to $v_{i-1}$ given $\{v_i, v_{i+1}\}$ in $G$. This motivates connecting $w$ to $v_{i-1}$ instead of $v_i$ in $H = \rgc$. Crucially, this new edge does not introduce any new d-connections, since $w$ and $v_{i-1}$ are in fact d-connected in $G$ given any conditioning set $Z$: if $Z$ contains a vertex of $C$, then $w \to v_i \leftarrow v_{i-1}$ is an active path, and if $Z$ and $C$ are disjoint, then $w \to v_i \to v_{i+1} \to \dots v_{i-1}$ (following the cycle) is an active path. By the same argument, $w$ and $v_i$ are d-connected given any set $Z$ in the new graph $H$, so deleting the edge $w \to v_i$ when going from $G$ to $H$ does not introduce any new d-separations. By a careful case analysis, one can extend this argumentation to a full proof of Markov equivalence between $G$ and $H$, see Appendix~\ref{apx:dcgs}.

\begin{figure}
    \captionsetup[subfigure]{labelformat=empty}
    \begin{subfigure}{\linewidth}
    \centering
    \tikz{
    \node (v1) at (0,0) {$v_{i-1}$};
    \node (v2) at (1.5,0) {$v_i$};
    \node (v3) at (2.5,-1) {$v_{i+1}$};
    \node (w) at (2, 1) {$w$};
    \node (x) at (0.5, -1) {$\dots$};

    \draw[->] (v1) to (v2);
    \draw[->] (v2) to (v3);
    \draw[->] (w) to (v2);
    \draw[->] (v3) to (x);
    \draw[->] (x) to (v1);
    }
    \caption{Cycle segment of $G$.}
    \end{subfigure}
    \begin{subfigure}{\linewidth}
    \centering
    \tikz{
    \node (v1) at (0,0) {$v_{i-1}$};
    \node (v2) at (1.5,0) {$v_i$};
    \node (v3) at (2.5,-1) {$v_{i+1}$};
    \node (w) at (2, 1) {$w$};
    \node (x) at (0.5, -1) {$\dots$};

    \draw[->] (v2) to (v1);
    \draw[->] (v3) to (v2);
    \draw[->] (w) to (v1);
    \draw[->] (x) to (v3);
    \draw[->] (v1) to (x);
    }
    \caption{Corresponding segment of the graph $H = \rgc$.}
    \end{subfigure}
    \caption{Illustration of the $\textsc{Reverse}$ construction.}
    \label{fig:constr}
\end{figure}

Now, suppose the DCG $G$ has multiple vertex-disjoint cycles. By construction, the operation $\textsc{Reverse}$ applied to any of the cycles, leaves all other cycles the same. Hence, one can obtain a new Markov-equivalent DCG for each combination of orientations of these cycles. This results in the following corollary:

\begin{corollary}\label{cor:DCG_MECs}
    Let $G$ be a directed cyclic graph that contains $k$ vertex-disjoint cycles of length at least $3$. Then, $|\MEC(G)| \geq 2^k$.
\end{corollary}

This is all we need to prove Theorem~\ref{thm:MECs}, part~\ref{item:dcg}.

\begin{proofof}{Theorem~\ref{thm:MECs}, part~\ref{item:dcg}}
    Let $G$ be a uniformly random DCG on $n$ vertices. Without loss of generality, we may assume that $n$ is divisible by $3$ and partition the vertex set $V$ of $G$ into disjoint sets $V_1, \dots, V_{n/3}$ of size $3$ each. Then, $\Pr(V_i \text{ induces a cycle of length $3$}) = 1/4$. Let $X$ denote the number of vertex sets $V_i$ that induce a cycle. We have $\E[X] = n/12$. By Corollary~\ref{cor:DCG_MECs} and Jensen's inequality, we conclude
    \begin{align*}
    \E\left[|\MEC(G)|\right] &\geq \E\left[2^{X}\right]\\ &\geq 2^{n/12}.
    \end{align*}
\end{proofof}

Again, we remark that it is straightforward to get a tight concentration bound on the number of vertex-disjoint cycles that we are using to bound $|\MEC(G)|$, using Chernoff's inequality. This means that almost all DCGs have exponentially large Markov equivalence classes.

\section{Discussion}\label{sec:Discussion}
Our results indicate significant underdetermination of causal structure from observational data in various settings. Particularly for DAGs, it might come as a surprise that previously known results on Markov equivalence classes being small for dense graphs do not extend to the sparse setting. However, our results are asymptotic in nature, calling for further numerical studies on how the expected size of Markov equivalence classes scales for small graphs with different edge densities. We hope that our theoretical results shed light on some of the sources of underdetermination in causal graphs: For DAGs, our proof of Theorem~\ref{thm:DAGs} could potentially be extended to show that edges are less likely to have determined directions when they appear between earlier layers of the tower decomposition. For ADMGs, we show that, on average, underdetermination arises frequently from small structures that induce underdetermined edges. This issue is addressed in the model class of maximal ancestral graphs (MAGs) through the maximality condition \citep{mags}, which forces underdetermined edges to be included in the graph and therefore makes the set of adjacencies unique again within a Markov equivalence class. As a result, MAGs may exhibit smaller Markov equivalence classes on average, which would be an interesting subject of further research (see \citet{mag-listing} for a first approach). However, the analysis becomes challenging as, to our knowledge, the enumeration of MAGs is an open problem. Note that our proof of Theorem~\ref{thm:MECs}, part~\ref{item:admg} also implies that under the uniform distribution, there is an expected super-exponential number of ADMGs corresponding to just a single MAG. In fact, most edges of a uniformly random ADMG are underdetermined with high probability, raising further questions about the informativeness of MAGs in this setting. For DCGs, addressing underdetermination seems to require some type of condition for directing cycles in light of our proof of Theorem~\ref{thm:MECs}, part~\ref{item:dcg}. 

Apart from ruling out certain graphs through additional assumptions, interventions also help to further distinguish between Markov-equivalent graphs. It would be interesting to explore the expected number of interventions required to uniquely recover a random causal graph (see \citet{interventional-MECs} for a first approach). Based on our results, this question becomes more significant for sparse DAGs, but also for ADMGs and DCGs, where it is not even known how to place interventions to efficiently split the Markov equivalence class. Finally, in practice, statistical uncertainty often prevents the identification of a single Markov equivalence class. Instead, researchers are faced not just with multiple graphs within one equivalence class, but often with sets of equivalence classes. An understanding of which equivalence classes are "close" to each other or how to separate equivalence classes in causal discovery more effectively, would be enormously useful.

\begin{acknowledgements} 
    This research was supported by NSF grant CCF-2321079. 

\end{acknowledgements}

\bibliography{refs}

\begin{thebibliography}{39}
\providecommand{\natexlab}[1]{#1}
\providecommand{\url}[1]{\texttt{#1}}
\expandafter\ifx\csname urlstyle\endcsname\relax
  \providecommand{\doi}[1]{doi: #1}\else
  \providecommand{\doi}{doi: \begingroup \urlstyle{rm}\Url}\fi

\bibitem[Chickering(1995)]{chickering95}
David~Maxwell Chickering.
\newblock A transformational characterization of equivalent bayesian network
  structures.
\newblock In \emph{Proceedings of the Eleventh Conference on Uncertainty in
  Artificial Intelligence}, UAI'95, page 87–98, San Francisco, CA, USA, 1995.
  Morgan Kaufmann Publishers Inc.
\newblock ISBN 1558603859.

\bibitem[Chickering(2002)]{Chickering-MEC-traversal}
David~Maxwell Chickering.
\newblock Learning equivalence classes of bayesian-network structures.
\newblock \emph{J. Mach. Learn. Res.}, 2:\penalty0 445–498, March 2002.
\newblock ISSN 1532-4435.
\newblock URL \url{https://doi.org/10.1162/153244302760200696}.

\bibitem[Chickering(2003)]{GES}
David~Maxwell Chickering.
\newblock Optimal structure identification with greedy search.
\newblock \emph{J. Mach. Learn. Res.}, 3\penalty0 (null):\penalty0 507–554,
  March 2003.
\newblock ISSN 1532-4435.
\newblock URL \url{https://doi.org/10.1162/153244303321897717}.

\bibitem[Chickering(2020)]{efficient-GES}
Max Chickering.
\newblock Statistically efficient greedy equivalence search.
\newblock In Jonas Peters and David Sontag, editors, \emph{Proceedings of the
  36th Conference on Uncertainty in Artificial Intelligence (UAI)}, volume 124
  of \emph{Proceedings of Machine Learning Research}, pages 241--249. PMLR,
  03--06 Aug 2020.
\newblock URL \url{https://proceedings.mlr.press/v124/chickering20a.html}.

\bibitem[Eggeling et~al.(2019)Eggeling, Viinikka, Vuoksenmaa, and
  Koivisto]{eggeling19a}
Ralf Eggeling, Jussi Viinikka, Aleksis Vuoksenmaa, and Mikko Koivisto.
\newblock On structure priors for learning bayesian networks.
\newblock In Kamalika Chaudhuri and Masashi Sugiyama, editors,
  \emph{Proceedings of the Twenty-Second International Conference on Artificial
  Intelligence and Statistics}, volume~89 of \emph{Proceedings of Machine
  Learning Research}, pages 1687--1695. PMLR, 16--18 Apr 2019.
\newblock URL \url{https://proceedings.mlr.press/v89/eggeling19a.html}.

\bibitem[Geiger et~al.(1990)Geiger, Verma, and Pearl]{Geiger90}
Dan Geiger, Thomas Verma, and Judea Pearl.
\newblock Identifying independence in bayesian networks.
\newblock \emph{Networks}, 20\penalty0 (5):\penalty0 507--534, 1990.
\newblock URL
  \url{https://onlinelibrary.wiley.com/doi/abs/10.1002/net.3230200504}.

\bibitem[Gillispie(2006)]{GILLISPIE2006}
Steven~B. Gillispie.
\newblock Formulas for counting acyclic digraph markov equivalence classes.
\newblock \emph{Journal of Statistical Planning and Inference}, 136\penalty0
  (4):\penalty0 1410--1432, 2006.
\newblock ISSN 0378-3758.
\newblock URL
  \url{https://www.sciencedirect.com/science/article/pii/S0378375804003982}.

\bibitem[Gillispie and Perlman(2001)]{Gillispie01}
Steven~B. Gillispie and Michael~D. Perlman.
\newblock Enumerating markov equivalence classes of acyclic digraph dels.
\newblock In \emph{Proceedings of the Seventeenth Conference on Uncertainty in
  Artificial Intelligence}, UAI'01, page 171–177, San Francisco, CA, USA,
  2001. Morgan Kaufmann Publishers Inc.
\newblock ISBN 1558608001.

\bibitem[Gillispie and Perlman(2002)]{GILLISPIE02}
Steven~B. Gillispie and Michael~D. Perlman.
\newblock The size distribution for markov equivalence classes of acyclic
  digraph models.
\newblock \emph{Artificial Intelligence}, 141\penalty0 (1):\penalty0 137--155,
  2002.
\newblock ISSN 0004-3702.
\newblock URL
  \url{https://www.sciencedirect.com/science/article/pii/S0004370202002643}.

\bibitem[He et~al.(2015)He, Jia, and Yu]{counting_MECS}
Yangbo He, Jinzhu Jia, and Bin Yu.
\newblock Counting and exploring sizes of markov equivalence classes of
  directed acyclic graphs.
\newblock \emph{Journal of Machine Learning Research}, 16\penalty0
  (79):\penalty0 2589--2609, 2015.
\newblock URL \url{http://jmlr.org/papers/v16/he15a.html}.

\bibitem[Heckerman et~al.(1995)Heckerman, Geiger, and
  Chickering]{heckerman1995learning}
David Heckerman, Dan Geiger, and David~M. Chickering.
\newblock Learning bayesian networks: The combination of knowledge and
  statistical data.
\newblock \emph{Machine Learning}, 20\penalty0 (3):\penalty0 197--243, 1995.
\newblock URL \url{https://doi.org/10.1023/A:1022623210503}.

\bibitem[Huang et~al.(2013)Huang, Li, Ye, Fleisher, Chen, Wu, and
  Reiman]{sparse-learning}
Shuai Huang, Jing Li, Jieping Ye, Adam Fleisher, Kewei Chen, Teresa Wu, and
  Eric Reiman.
\newblock A sparse structure learning algorithm for gaussian bayesian network
  identification from high-dimensional data.
\newblock \emph{IEEE Transactions on Pattern Analysis and Machine
  Intelligence}, 35\penalty0 (6):\penalty0 1328--1342, 2013.
\newblock ISSN 0162-8828.
\newblock URL \url{https://doi.org/10.1109/TPAMI.2012.129}.

\bibitem[Kalisch and B{{\"u}}hlmann(2007)]{efficient_PC}
Markus Kalisch and Peter B{{\"u}}hlmann.
\newblock Estimating high-dimensional directed acyclic graphs with the
  pc-algorithm.
\newblock \emph{Journal of Machine Learning Research}, 8\penalty0
  (22):\penalty0 613--636, 2007.
\newblock URL \url{http://jmlr.org/papers/v8/kalisch07a.html}.

\bibitem[Katz et~al.(2019)Katz, Shanmugam, Squires, and
  Uhler]{interventional-MECs}
Dmitriy Katz, Karthikeyan Shanmugam, Chandler Squires, and Caroline Uhler.
\newblock Size of interventional markov equivalence classes in random dag
  models.
\newblock In Kamalika Chaudhuri and Masashi Sugiyama, editors,
  \emph{Proceedings of the Twenty-Second International Conference on Artificial
  Intelligence and Statistics}, volume~89 of \emph{Proceedings of Machine
  Learning Research}, pages 3234--3243. PMLR, 16--18 Apr 2019.
\newblock URL \url{https://proceedings.mlr.press/v89/katz19a.html}.

\bibitem[Koster(1999)]{Koster_ADMGs}
Jan T.~A. Koster.
\newblock On the validity of the markov interpretation of path diagrams of
  gaussian structural equations systems with correlated errors.
\newblock \emph{Scandinavian Journal of Statistics}, 26\penalty0 (3):\penalty0
  413--431, 1999.
\newblock URL
  \url{https://onlinelibrary.wiley.com/doi/abs/10.1111/1467-9469.00157}.

\bibitem[Kuipers and Moffa(2015)]{Kuipers2015}
Jack Kuipers and Giusi Moffa.
\newblock Uniform random generation of large acyclic digraphs.
\newblock \emph{Statistics and Computing}, 25\penalty0 (2):\penalty0 227--242,
  2015.
\newblock URL \url{https://doi.org/10.1007/s11222-013-9428-y}.

\bibitem[Kuipers and Moffa(2017)]{Kuipers02012017}
Jack Kuipers and Giusi Moffa.
\newblock Partition mcmc for inference on acyclic digraphs.
\newblock \emph{Journal of the American Statistical Association}, 112\penalty0
  (517):\penalty0 282--299, 2017.
\newblock URL \url{https://doi.org/10.1080/01621459.2015.1133426}.

\bibitem[Mitzenmacher and Upfal(2005)]{mitzenmacher2005probability}
Michael Mitzenmacher and Eli Upfal.
\newblock \emph{Probability and Computing: Randomized Algorithms and
  Probabilistic Analysis}.
\newblock Cambridge University Press, Cambridge, 2005.
\newblock ISBN 978-0-521-83540-4.
\newblock URL \url{https://doi.org/10.1017/CBO9780511813603}.

\bibitem[Pearl(2009)]{Pea09}
J.~Pearl.
\newblock \emph{Causality}.
\newblock Cambridge, 2nd edition, 2009.

\bibitem[Radhakrishnan et~al.(2017)Radhakrishnan, Solus, Uhler, and
  Wainwright]{counting_MECs_2}
Adityanarayanan Radhakrishnan, Liam Solus, Caroline Uhler, and Martin~J.
  Wainwright.
\newblock Counting markov equivalence classes by number of immoralities.
\newblock In \emph{Proceedings of the 33rd Conference on Uncertainty in
  Artificial Intelligence (UAI)}, Sydney, Australia, August 11--15 2017. AUAI
  Press.

\bibitem[Raskutti and Uhler(2018)]{GPS}
Garvesh Raskutti and Caroline Uhler.
\newblock Learning directed acyclic graph models based on sparsest
  permutations.
\newblock \emph{Stat}, 7\penalty0 (1):\penalty0 e183, 2018.
\newblock URL \url{https://onlinelibrary.wiley.com/doi/abs/10.1002/sta4.183}.

\bibitem[Richardson(1996{\natexlab{a}})]{richardson-cycles}
Thomas Richardson.
\newblock A polynomial-time algorithm for deciding markov equivalence of
  directed cyclic graphical models.
\newblock In \emph{Proceedings of the Twelfth International Conference on
  Uncertainty in Artificial Intelligence}, UAI'96, page 462–469, San
  Francisco, CA, USA, 1996{\natexlab{a}}. Morgan Kaufmann Publishers Inc.
\newblock ISBN 155860412X.

\bibitem[Richardson(1996{\natexlab{b}})]{richardson2018discovering}
Thomas Richardson.
\newblock Discovering cyclic causal structure.
\newblock 1996{\natexlab{b}}.
\newblock URL \url{https://doi.org/10.1184/R1/6491426.v1}.

\bibitem[Richardson(1997)]{characterizing_MECs_in_DCGs}
Thomas Richardson.
\newblock A characterization of markov equivalence for directed cyclic graphs.
\newblock \emph{International Journal of Approximate Reasoning}, 17\penalty0
  (2):\penalty0 107--162, 1997.
\newblock ISSN 0888-613X.
\newblock URL
  \url{https://www.sciencedirect.com/science/article/pii/S0888613X97000200}.
\newblock Uncertainty in AI (UAI'96) Conference.

\bibitem[Richardson and Spirtes(2002)]{mags}
Thomas Richardson and Peter Spirtes.
\newblock {Ancestral graph Markov models}.
\newblock \emph{The Annals of Statistics}, 30\penalty0 (4):\penalty0 962 --
  1030, 2002.
\newblock URL \url{https://doi.org/10.1214/aos/1031689015}.

\bibitem[Schmid and Sly(2024)]{SS24}
Dominik Schmid and Allan Sly.
\newblock On the number and size of markov equivalence classes of random
  directed acyclic graphs, 2024.
\newblock URL \url{https://arxiv.org/abs/2209.04395}.

\bibitem[Spirtes(1995)]{DCGs_d-separation}
Peter Spirtes.
\newblock Directed cyclic graphical representations of feedback models.
\newblock In \emph{Proceedings of the Eleventh Conference on Uncertainty in
  Artificial Intelligence}, UAI'95, page 491–498, San Francisco, CA, USA,
  1995. Morgan Kaufmann Publishers Inc.
\newblock ISBN 1558603859.

\bibitem[Spirtes and Richardson(1997)]{MECs_in_ADMGs}
Peter Spirtes and Thomas~S. Richardson.
\newblock A polynomial time algorithm for determining dag equivalence in the
  presence of latent variables and selection bias.
\newblock In David Madigan and Padhraic Smyth, editors, \emph{Proceedings of
  the Sixth International Workshop on Artificial Intelligence and Statistics},
  volume~R1 of \emph{Proceedings of Machine Learning Research}, pages 489--500.
  PMLR, 04--07 Jan 1997.
\newblock URL \url{https://proceedings.mlr.press/r1/spirtes97b.html}.
\newblock Reissued by PMLR on 30 March 2021.

\bibitem[Spirtes et~al.(1998)Spirtes, Richardson, Meek, Scheines, and
  Glymour]{Spirtes_ADMGs}
Peter Spirtes, Thomas Richardson, Christopher Meek, Richard Scheines, and Clark
  Glymour.
\newblock Using path diagrams as a structural equation modeling tool.
\newblock \emph{Sociological Methods \& Research}, 27\penalty0 (2):\penalty0
  182--225, 1998.
\newblock URL \url{https://doi.org/10.1177/0049124198027002003}.

\bibitem[Spirtes et~al.(2001)Spirtes, Glymour, and Scheines]{PC}
Peter Spirtes, Clark Glymour, and Richard Scheines.
\newblock \emph{Causation, Prediction, and Search}.
\newblock The MIT Press, 01 2001.
\newblock ISBN 9780262284158.
\newblock URL \url{https://doi.org/10.7551/mitpress/1754.001.0001}.

\bibitem[Spirtes et~al.(1999)Spirtes, Meek, and Richardson]{FCI}
Peter~L. Spirtes, Christopher Meek, and Thomas~S. Richardson.
\newblock An algorithm for causal inference in the presence of latent variables
  and selection bias.
\newblock In \emph{Computation, Causation, and Discovery}. AAAI Press, 05 1999.
\newblock ISBN 9780262315821.
\newblock URL \url{https://doi.org/10.7551/mitpress/2006.003.0009}.

\bibitem[Steinsky(2003)]{STEINSKY2003}
Bertran Steinsky.
\newblock Enumeration of labelled chain graphs and labelled essential directed
  acyclic graphs.
\newblock \emph{Discrete Mathematics}, 270\penalty0 (1):\penalty0 267--278,
  2003.
\newblock ISSN 0012-365X.
\newblock URL
  \url{https://www.sciencedirect.com/science/article/pii/S0012365X02008385}.

\bibitem[Steinsky(2013)]{Steinsky2013}
Bertran Steinsky.
\newblock Enumeration of labelled essential graphs.
\newblock \emph{Ars Combinatoria}, 111:\penalty0 485--494, 2013.
\newblock URL
  \url{https://combinatorialpress.com/article/ars/Volume%20111/volume-111-paper-40.pdf}.

\bibitem[Talvitie and Koivisto(2019)]{asking_our_question}
Topi Talvitie and Mikko Koivisto.
\newblock Counting and sampling markov equivalent directed acyclic graphs.
\newblock \emph{Proceedings of the AAAI Conference on Artificial Intelligence},
  33\penalty0 (01):\penalty0 7984--7991, Jul. 2019.
\newblock URL \url{https://ojs.aaai.org/index.php/AAAI/article/view/4799}.

\bibitem[Talvitie et~al.(2020)Talvitie, Vuoksenmaa, and Koivisto]{talvitie}
Topi Talvitie, Aleksis Vuoksenmaa, and Mikko Koivisto.
\newblock Exact sampling of directed acyclic graphs from modular distributions.
\newblock In Ryan~P. Adams and Vibhav Gogate, editors, \emph{Proceedings of The
  35th Uncertainty in Artificial Intelligence Conference}, volume 115 of
  \emph{Proceedings of Machine Learning Research}, pages 965--974. PMLR, 22--25
  Jul 2020.
\newblock URL \url{https://proceedings.mlr.press/v115/talvitie20a.html}.

\bibitem[Verma and Pearl(1990)]{d-separation}
Thomas Verma and Judea Pearl.
\newblock Causal networks: Semantics and expressiveness.
\newblock In Ross~D. SHACHTER, Tod~S. LEVITT, Laveen~N. KANAL, and John~F.
  LEMMER, editors, \emph{Uncertainty in Artificial Intelligence}, volume~9 of
  \emph{Machine Intelligence and Pattern Recognition}, pages 69--76.
  North-Holland, 1990.
\newblock URL
  \url{https://www.sciencedirect.com/science/article/pii/B9780444886507500111}.

\bibitem[Verma and Pearl(1991)]{verma-pearl}
TS~Verma and Judea Pearl.
\newblock Equivalence and synthesis of causal models.
\newblock In M.~Henrion, R.~Shachter, L.~Kanal, and J.~Lemmer, editors,
  \emph{Proceedings of the 6th Conference on Uncertainty in Artificial
  Intelligence}, page 220–227, 1991.

\bibitem[Wang et~al.(2024)Wang, Du, and Zhou]{mag-listing}
Tian-Zuo Wang, Wen-Bo Du, and Zhi-Hua Zhou.
\newblock An efficient maximal ancestral graph listing algorithm.
\newblock In Ruslan Salakhutdinov, Zico Kolter, Katherine Heller, Adrian
  Weller, Nuria Oliver, Jonathan Scarlett, and Felix Berkenkamp, editors,
  \emph{Proceedings of the 41st International Conference on Machine Learning},
  volume 235 of \emph{Proceedings of Machine Learning Research}, pages
  50353--50378. PMLR, 21--27 Jul 2024.
\newblock URL \url{https://proceedings.mlr.press/v235/wang24o.html}.

\bibitem[Wien\"{o}bst et~al.(2023)Wien\"{o}bst, Bannach, and
  Li\'{s}kiewicz]{counting_MECs_3}
Marcel Wien\"{o}bst, Max Bannach, and Maciej Li\'{s}kiewicz.
\newblock Polynomial-time algorithms for counting and sampling markov
  equivalent dags with applications.
\newblock \emph{Journal of Machine Learning Research}, 24\penalty0
  (213):\penalty0 1--45, 2023.
\newblock URL \url{http://jmlr.org/papers/v24/22-0495.html}.

\end{thebibliography}

\newpage

\onecolumn

\title{Proofs\\(Supplementary Material)}
\maketitle

\appendix

\section{Proofs for Sparse DAGs} \label{apx:dags}

Here, we restate and prove our technical results on random DAGs sampled from $D(n,p)$.

{
\renewcommand{\thetheorem}{\ref{lem:tower-distr}}
\begin{lemma}
    Let $G \sim D(n,p)$, and fix a tower vector $h = (h_1, \dots, h_s)$ with entries that sum up to $n$, and a tower decomposition $H = (H_1 \dots, H_s)$ with $|H_i| = h_i$. Then, the following holds:
    
    \begin{enumerate}[label=(\alph*)]
        \item the probability that $h$ is the tower vector of $G$ is proportional to 
    \begin{align*}
        w(h) = \binom{n}{h_1, \dots, h_s} \prod_{k=2}^s\frac{\left(1 - (1-p)^{h_{k-1}}\right)^{h_k}}{(1-p)^{h_k \sum_{i=1}^{k-1}h_i}}; 
    \end{align*}
        \item conditional on the event that $H$ is the tower decomposition of $G$, the parent sets $\Pa(v)$ are independently distributed for all $v \in V(G)$, and for $v \in H_j$ and $S \subseteq \bigcup_{i=1}^{j-1}H_i$, we have 
        \begin{align*}
            \Pr\left(\Pa(v) = S\right) = \begin{cases}
                C \cdot \left(\frac{p}{1-p}\right)^{|S|} \text{ if } S \cap H_{j-1} \neq \emptyset\\
                0 \text{ otherwise,}
            \end{cases}
        \end{align*}
        where $C$ is the normalization constant.
    \end{enumerate}
\end{lemma}
\addtocounter{theorem}{-1}
}

\begin{proof}
    Fix a tower vector $h = (h_1, \dots, h_s)$ with entries that sum up to $n$ and a tower decomposition $H = (H_1, \dots, H_s)$ with $|H_k| = h_k$. The key observation is that a DAG $G$ has tower decomposition $H$ if and only if for each vertex $v \in H_k$, the parents of $v$ are a subset of $\bigcup_{i = 1}^{k-1} H_i$ that intersects $H_{k-1}$ (Lemma 3 of \citet{talvitie}). We define 
    \begin{align*}
        \mathcal{C}_k = \{S \subseteq \bigcup_{i = 1}^{k-1} H_i \mid S \cap H_{k-1} \neq \emptyset\}; \;
        \mathcal{C}_{k,s} = \{S \in \cC_k \mid |S| = s\}, 
    \end{align*}
    setting $H_0 = \emptyset$ to cover the case $k = 1$. In particular, given that $H$ is the tower decomposition of $G$, the parent sets of each vertex in $H_k$ can be chosen independently in $\mathcal{C}_k$, which together with Definition~\ref{defn:Dnp} of $D(n,p)$ implies part~\ref{item:b} of the Lemma. Moreover, we get that the probability for $G \sim D(n,p)$ to have tower decomposition $H$ is proportional to
    \begin{align*}
       w(H) &:= \sum_{G: H(G) = H} \prod_{k = 1}^s \prod_{v \in H_k} \left(\frac{p}{1-p}\right)^{|\Pa(v)|}\\
        &= \prod_{k = 1}^s \prod_{v \in H_k} \sum_{P \in \cC_k} \left(\frac{p}{1-p}\right)^{|P|}  \\ 
        &= \prod_{k = 1}^s \prod_{v \in H_k} \sum_{s=0}^n \left(\frac{p}{1-p}\right)^{s} \cdot |\cC_{k,s}|\\
        &= \prod_{k = 2}^s \prod_{v \in H_k} \sum_{s=0}^n \left(\frac{p}{1-p}\right)^{s} \cdot \left(\binom{\sum_{i=1}^{k-1}h_i}{s} - \binom{\sum_{i=1}^{k-2}h_{i}}{s}\right)\\
        &= \prod_{k = 2}^s \prod_{v \in H_k} \left( \left(1 + \frac{p}{1-p}\right)^{\sum_{i=1}^{k-1}h_i} - \left(1 + \frac{p}{1-p}\right)^{\sum_{i=1}^{k-2}h_i}\right)\\
        &= \prod_{k = 2}^s (1-p)^{-h_k \sum_{i=1}^{k-1}h_i} \cdot \left(1 - (1-p)^{h_{k-1}}\right)^{h_k}. 
    \end{align*}

Now part~\ref{item:a} of Lemma~\ref{lem:tower-distr} follows from summing over all $\binom{n}{h_1, \dots, h_s}$ choices of tower decompositions with layer sizes $h_1, \dots, h_s$. 
\end{proof}

{
\renewcommand{\thetheorem}{\ref{lem:upper-tailbd}}
\begin{lemma}
    Let $p = o\left(\frac{1}{\log n}\right)$ and $G \sim D(n,p)$. Then, the number of sources in $G$ is less than or equal to $\frac{5}{p}$ with high probability. 
\end{lemma}
\addtocounter{theorem}{-1}
}

\begin{proof}
    Let $p = o\left(\frac{1}{\log n}\right)$. Without loss of generality, we assume that $p^{-1}$ is an integer. For each index $j$, we define $\cH_j$ to be the set of tower vectors $h=(h_1, \dots, h_s)$ with entries that sum up to $n$ and $h_j \geq 5p^{-1} \geq h_{j+1}$. Consider the functions $\varphi_j$ that map $h=(h_1, \dots, h_s) \in \cH_j$ to the vector $\varphi_j(h) = (h_1, \dots, h_{j-1}, \ell_1, \dots, \ell_r, h_{j+1}, \dots, h_s)$, where $r$ is equal to $h_j p$ rounded to the nearest integer, $\ell_2 = \ell_3 = \dots = \ell_r = p^{-1}$ and $\ell_1 = h_j - \sum_{i=2}^r \ell_i$. Note that the entries of $\varphi_j(h)$ are still positive integers that sum up to $n$, so it is a valid tower vector. Moreover, $\ell_1 \in [p^{-1}/2, 3p^{-1}/2]$. We are now going to bound the ratio of the weights of $h$ and $\varphi_j(h)$. By Lemma~\ref{lem:tower-distr},
    \begin{align*}
        \frac{w(h)}{w(\varphi_j(h))} = \frac{\prod_{i=1}^r\ell_i!}{h_j!} (1-p)^{\sum_{1\leq i<k\leq r}\ell_i \ell_k} \cdot \frac{(1-(1-p)^{h_{j-1}})^{h_j-\ell_1}}{\prod_{i=2}^r (1-(1-p)^{\ell_{i-1}})^{\ell_i}} \cdot \left(\frac{1-(1-p)^{h_j}}{1-(1-p)^{\ell_r}}\right)^{h_{j+1}}.
    \end{align*}
    For the edge cases of this formula to be correct, we set $h_0 = \infty$ and $h_{s+1} = 0$. We bound each term separately: First, since $\sum_{i=1}^r \ell_i = h_j$, we have $\prod_{i=1}^r \ell_i! \leq h_j!$.\\
    Secondly, since $r \geq 5$, we get the estimate
    \begin{align*}
        \sum_{1 \leq i < k \leq r}\ell_i\ell_k &= \ell_r (h_j - \ell_r) + \ell_{r-1} ( h_j - \ell_r - \ell_{r-1}) + \sum_{1 \leq i < k \leq r-2} \ell_i\ell_k\\ &= 2h_jp^{-1} - 3p^{-2} + \sum_{1 \leq i < k \leq r-2} \ell_i\ell_k \geq 2h_jp^{-1}.
    \end{align*}
    Thirdly, since $\ell_i \geq (2p)^{-1}$, we may use 
    \begin{align*}
        1-(1-p)^{\ell_i} \geq 1-(1-p)^{1/(2p)} \geq 1 - e^{-1/2} \geq 7/20,
    \end{align*}
    and
    \begin{align*}
        1 - (1-p)^{\ell_r} = 1-(1-p)^{1/p} \geq 1-e^{-1} \geq 1/2.
    \end{align*}
    Finally, using the condition $h_{j} \geq 5p^{-1} \geq  h_{j+1}$, we obtain
    \begin{align*}
        \frac{w(h)}{w(\varphi_j(h))} \leq (1-p)^{2h_jp^{-1}} \cdot (20/7)^{\sum_{i=2}^r\ell_i} \cdot 2^{h_{j+1}} \leq e^{-2h_j + \log(20/7)\cdot h_j + \log(2)\cdot h_j} \leq e^{-h_j/4} \leq e^{-p^{-1}}.
    \end{align*}
    Now, let $G \sim D(n,p)$. Note that each function $\varphi_j$ maps at most $np$ vectors to the same image (the extreme case is when $\varphi_1(h) = (p^{-1}, \dots, p^{-1})$). We calculate
    \begin{align*}
        \Pr\left(h_{j}(G) \geq 5p^{-1} \geq  h_{j+1}(G)\right) = C \cdot \sum_{h \in \mathcal{H}_j}w(h) \leq C \cdot \sum_{h \in \mathcal{H}_j}w(\varphi_j(h))\cdot e^{-p^{-1}} \leq np \cdot e^{-p^{-1}},
    \end{align*}
    where $C = \left(\sum_h w(h)\right)^{-1}$ is the normalization constant. We conclude
    \begin{align*}
        \Pr\left(h_1(G) \geq 5p^{-1}\right) &\leq \Pr\left(h_{1}(G) \geq 5p^{-1} \geq  h_{2}(G)\right) + \Pr\left(h_2(G) \geq 5p^{-1}\right)\leq \dots\\ &\leq \sum_{i=1}^{T(G)-1}\Pr\left(h_{j}(G) \geq 5p^{-1} \geq  h_{j+1}(G)\right) \leq n^2p \cdot e^{-p^{-1}} = o(1).
    \end{align*}

\end{proof}

We will need the following bound in the proof of the next Lemma:

\begin{lemma}\label{lem:multinomial}
    Let $h_1, \dots, h_r$ be positive integers that sum up to $n$. Then, 
    \begin{align*}
        \binom{n}{h_1, \dots, h_r} \leq \frac{n^n}{ h_r^{h_1} \cdot \prod_{i=2}^rh_{i-1}^{h_i}}
    \end{align*}
\end{lemma}

\begin{proof}
    The lemma simply follows from: 
    \begin{align*}
        n^n = \left(\sum_{i=1}^r h_i\right)^n = \sum_{j_1 + \dots + j_r = n} \binom{n}{j_1, \dots, j_r}h_1^{j_1} \cdot \ldots \cdot h_r^{j_r} \geq \binom{n}{h_2, \dots, h_r, h_1}h_1^{h_2} \cdot \ldots \cdot h_{r-1}^{h_r} \cdot h_r^{h_1}.
    \end{align*}
\end{proof}

{
\renewcommand{\thetheorem}{\ref{lem:lower-tailbd}}
\begin{lemma}
    Let $\frac{6}{n} \leq p \leq o\left(\frac{1}{\log n}\right)$ and $G \sim D(n,p)$. Then, with high probability, we have $h_1(G) \geq \frac{p^{-1}}{20 \log(p^{-1})}$ and $h_2(G) \geq \frac{p^{-1}}{\log^2(p^{-1})}$.
\end{lemma}
\addtocounter{theorem}{-1}
}

\begin{proof}
    Let $6/n \leq p \leq o\left(\frac{1}{\log n}\right)$. First, we define $\cH_1$ to be the set of tower vectors $h= (h_1, \dots, h_s)$ with entries that sum up to $n$ and $h_1 \leq p^{-1}/(20\log (p^{-1}))$. Consider the function $\varphi_1$ that maps a tower vector $h = (h_1, \dots, h_s) \in \cH_1$ to the vector $\varphi_1(h) = (L, h_{r+1}, \dots, h_s)$, where $r = \min \{k:\sum_{i=1}^{k} h_i > p^{-1}/5\}$ and $L = \sum_{i=1}^{r} h_i$. Note that $\varphi_1(h)$ is again a valid tower vector, and by Lemma~\ref{lem:tower-distr}, we get
    \begin{align}
        \notag \frac{w(h)}{w(\varphi_1(h))} &= \frac{L!}{\prod_{i=1}^{r} h_i!} (1-p)^{-\sum_{1\leq i<k\leq r}h_i h_k} \cdot \prod_{i=2}^{r} (1-(1-p)^{h_{i-1}})^{h_i} \cdot \left(\frac{1-(1-p)^{h_{r}}}{1-(1-p)^{L}}\right)^{h_{r+1}}\\
        &\leq \frac{L!}{\prod_{i=1}^{r} h_i!} (1-p)^{-\sum_{1\leq i<k\leq r}h_i h_k} \cdot \prod_{i=2}^{r} (1-(1-p)^{h_{i-1}})^{h_i}. \label{eq1}
    \end{align}
    We are interested in further bounding expression \eqref{eq1} under the constraint $L = \sum_{i=1}^r h_i$. First of all, note that increasing $h_r$ by $1$, while leaving $h_1, \dots, h_{r-1}$ the same, changes expression \eqref{eq1} by a factor of 
    \begin{align*}
        \frac{L+1}{h_r + 1} \cdot (1-p)^{-\sum_{i=1}^{r-1}h_i} \cdot (1-(1-p)^{h_{r-1}}) \leq \frac{L}{h_r} \cdot (1-p)^{-p^{-1}/5}  \cdot p h_{r-1} \leq \frac{L}{h_r} \cdot e^{1/5} \cdot \frac{1}{5}.
    \end{align*}
    Here, we used the fact that $\sum_{i=1}^{r-1}h_i \leq p^{-1}/5$. This factor is strictly less than $1$ if, $h_r > L/4$, hence, we may assume $h_r \leq L/4$, and therefore $L \leq 4 p^{-1}/15$. Now, we can bound \eqref{eq1} using Lemma~\ref{lem:multinomial}:
    \begin{align*}
        \frac{w(h)}{w(\varphi_1(h))} &\leq \frac{L^L}{h_r^{h_1} \cdot \prod_{i=2}^{r}h_{i-1}^{h_i}} (1-p)^{-L^2} \prod_{i=2}^r (h_{i-1}p)^{h_i}\\ &\leq L^L e^{pL^2}p^{L-h_1} \leq \left(\frac{4}{15}\right)^Lp^{-L}e^{4L/15}p^{L}p^{p^{-1}/(20\log p)}\\
        &\leq e^{L(\log(4/15) + 4/15)) + p^{-1}/20} \leq e^{p^{-1}/5 \cdot (\log(4/15)+4/15+1/4)} \leq \left(\frac{5}{11}\right)^{p^{-1}/5}.
    \end{align*}
    
    Note that for a fixed tower vector of the form $ \varphi_1(h) = (L, h_{r+1}, \dots, h_s)$, the entries $h_1, \dots, h_{r-1}$ of the preimage must sum up to a number less than $p^{-1}/5$, for which there are less than $2^{p^{-1}/5}$ choices. After choosing $h_1, \dots, h_{r-1}$, the entry $h_r = L - \sum_{i=1}^{r-1}$ and the rest of the preimage is uniquely defined. Hence, the function $\varphi_1$ sends at most $2^{p^{-1}/5}$ tower vectors to the same image. Now, let $G \sim D(n,p)$ and $C = \left(\sum_h w(h)\right)^{-1}$. We get
    \begin{align*}
        \Pr\left(h_{1}(G) \leq \frac{p^{-1}}{20 \log(p^{-1})}\right) &= C \cdot \sum_{h \in \mathcal{H}_1}w(h) \leq C \cdot \sum_{h \in \mathcal{H}_1}w(\varphi_1(h))\cdot \left(\frac{5}{11}\right)^{p^{-1}/5}\\ &\leq 2^{p^{-1}/5} \cdot \left(\frac{5}{11}\right)^{p^{-1}/5} = \left(\frac{10}{11}\right)^{p^{-1}/5} = o(1).
    \end{align*}

    Now, let $\cH_2$ be the set of tower vectors $h = (h_1, \dots, h_s)$ with the following properties: 
    \begin{enumerate}[label = (\alph*)]
        \item the entries $h_i$ sum up to $n$;
        \item $\frac{p^{-1}}{20 \log (p^{-1})} \leq h_1 \leq 5p^{-1}$;
        \item $h_2 \leq \frac{p^{-1}}{\log^2(p^{-1})}$.
    \end{enumerate}
    Consider the function $\varphi_2$ that maps a tower vector $h = (h_1, \dots, h_s) \in \cH_2$ to the vector $\varphi_2(h) = (h_1, L, h_{r+1}, \dots, h_s)$, where $r = \min \{k:\sum_{i=2}^{k} h_i > p^{-1}/(800\log(p^{-1}))\}$ and $L = \sum_{i=2}^{r} h_i$. This is well-defined, as $p \geq 6/n$, so $\sum_{i=2}^s h_i \geq n - 5p^{-1} \geq n/6 > p^{-1}/(800\log(p^{-1}))$. Moreover, $\varphi_2(h)$ is again a valid tower vector, and by Lemma~\ref{lem:tower-distr}, we get
    \begin{align*}
        \frac{w(h)}{w(\varphi_2(h))} \leq \frac{L!}{\prod_{i=2}^{r} h_i!} (1-p)^{-\sum_{2\leq i<k\leq r}h_i h_k} \cdot \frac{\prod_{i=2}^{r} (1-(1-p)^{h_{i-1}})^{h_i}}{(1-(1-p)^{h_1})^{L}}.
    \end{align*}
    Here, increasing $h_r$ by $1$, while leaving $h_1, \dots, h_{r-1}$ the same, changes the expression above by a factor of 
    \begin{align*}
        \frac{L+1}{h_r + 1} \cdot (1-p)^{-\sum_{i=2}^{r-1}h_i} \cdot \frac{1-(1-p)^{h_{r-1}}}{1-(1-p)^{h_1}} \leq \frac{L}{h_r} \cdot (1-p)^{-p^{-1}/800}  \cdot \frac{6h_{r-1}}{h_1} \leq \frac{L}{h_r} \cdot \frac{1}{4}.
    \end{align*}
    Here, we used the Bernoulli-type inequalities $1-pm \leq (1-p)^m \leq 1-pm/(1+pm)$ together with the fact $h_1 \leq 5p^{-1}$. The factor above is strictly less than $1$ if, $h_r > L/4$, hence, we may assume $h_r \leq L/4$, and therefore $L \leq p^{-1}/(600 \log(p^{-1}))$. Reusing the bound on \eqref{eq1} from before, we get:
    \begin{align*}
        \frac{w(h)}{w(\varphi_1(h))} &\leq L^L e^{pL^2}\left(\frac{6p}{h_1p}\right)^{L-h_2} \leq L^L e^{L/600} \left(\frac{6}{30L}\right)^{L-h_2}\\ &\leq \left(\frac{1}{4}\right)^L \cdot e^{\log(5L)\cdot h_2} \leq \left(\frac{1}{4}\right)^{p^{-1}/(800 \log(p^{-1}))} e^{\log(5L) \cdot p^{-1} / \log^2(p^{-1})} \leq \left(\frac{1}{3}\right)^{p^{-1}/(800 \log(p^{-1}))},
    \end{align*}
    where the last step follows when $p$ is small enough. 
    Now, the function $\varphi_2$ sends at most $2^{p^{-1}/(800\log(p^{-1}))}$ tower vectors to the same image. Let $G \sim D(n,p)$ and $C = \left(\sum_h w(h)\right)^{-1}$. We conclude using Lemma~\ref{lem:upper-tailbd} and the result above:
    \begin{align*}
        \Pr\left(h_{2}(G) \leq \frac{p^{-1}}{\log^2(p^{-1})}\right) &\leq \Pr\left(h_{2}(G) \leq \frac{p^{-1}}{\log^2(p^{-1})} \wedge \frac{p^{-1}}{20 \log (p^{-1})} \leq h_1 \leq 5p^{-1}\right)\\
        &+ \Pr\left(h_1(G) < \frac{p^{-1}}{20\log(p^{-1})}\right) + \Pr\left(h_1(G) > 5p^{-1}\right)\\
        &\leq C \cdot \sum_{h \in \mathcal{H}_2}w(h) + o(1)\\ 
        &\leq C \cdot \sum_{h \in \mathcal{H}_2}w(\varphi_2(h)) \cdot \left(\frac{1}{3}\right)^{p^{-1}/(800 \log(p^{-1}))} + o(1) \\
        &\leq \left(\frac{2}{3}\right)^{p^{-1}/(800 \log(p^{-1}))} + o(1) = o(1).
    \end{align*}
\end{proof}

In the final proof of this section, we will make use of the Chernoff bound:

\begin{lemma}{(Chernoff bound, see \citet{mitzenmacher2005probability}, Theorem 4.5)} \label{lem:chernoff}
    Let $X_1, \dots, X_n$ be independent random variables with $\Pr(X_i=1) = p$ and $\Pr(X_i = 0) = 1-p$ and define $X = \sum_{i=1}^n X_i$. Then, for any $\delta \in (0,1)$, we have 
    \begin{align*}
        \Pr\left(X \leq (1-\delta)np\right) \leq e^{-np\delta^2/2}.
    \end{align*}
\end{lemma}

{
\renewcommand{\thetheorem}{\ref{prop:DAGs}}
\begin{prop}
    Let $6/n \leq p \leq o\left(\frac{1}{\log n}\right)$ and $G \sim D(n,p)$. Then, with probability $1-o(1)$, $G$ has a matching of reversible layer-2-edges of size at least $\frac{p^{-1}}{16e^5\log^2(p^{-1})}$.
\end{prop}
\addtocounter{theorem}{-1}
}

\begin{proof}
    Let $6/n \leq p \leq o\left(\frac{1}{\log n}\right)$ and $G \sim D(n,p)$. For technical reasons that will become apparent later, let us define $\widetilde{H_2}(G)$ to be a subset of $H_2(G)$ of at most $h_1(G)/2$ vertices. This subset can be picked by an arbitrary but fixed rule (for instance, take the $h_1(G)/2$ vertices in $H_2(G)$ with the smallest indices, according to some indexing of the vertex set $V(G)$), and $\widetilde{H_2}(G) = H_2(G)$ in the case $h_2(G) \leq h_1(G) / 2$. Now, let $S(G) \subseteq \widetilde{H_2}(G)$ be the set of vertices in $\widetilde{H_2}(G)$ that have exactly one parent. Furthermore, let $X(G)$ be the set of vertices in $S$ whose parent is different from all the other parents of vertices in $S$. By definition, the edges that connect vertices in $X(G)$ to their parents form a matching of reversible layer-2-edges. Hence, it suffices to bound $|X(G)|$. First, note that conditional on the tower decomposition $H(G) = H = (H_1, \dots, H_s)$ and on the event $S(G) = S$, the parent vertex of each vertex $v \in S$ is distributed uniformly and independently in $H_1$. Uniformity follows from the symmetry of the $D(n,p)$ distribution with respect to permuting vertices, and independence follows from Lemma~\ref{lem:tower-distr}, part~\ref{item:b}. By a union bound, we get for $v \in S$:
    \begin{align*}
        &\Pr\left(v \in X(G) \mid H(G) = H, S(G) = S\right)\\ &\geq 1 - \sum_{w \in S, w \neq v} \Pr\left(\Pa(v) = \Pa(w)\right)\\ &= 1 - \frac{|S|}{h_1} \geq 1 - \frac{|\widetilde{H_2}|}{h_1} \geq \frac{1}{2}.
    \end{align*}
    Hence, $\E[|X(G)| \mid H(G) = H, S(G) = S] \geq |S|/2$, and by a standard Chernoff concentration bound (see Lemma~\ref{lem:chernoff}), we get 
    \begin{align}
        \notag &\Pr\left(|X(G)| \leq |S|/4 \mid H(G) = H, S(G) = S\right)\\
        &\leq e^{-|S|/16}. \label{eq:conc-X}
    \end{align}
    Next, we turn our attention to bounding $|S(G)|$. Conditional on the tower distribution $H(G) = H = (H_1, \dots, H_s)$, we have for each vertex $v \in \widetilde{H_2}(G)$ by Lemma~\ref{lem:tower-distr}, part~\ref{item:b}:
    \begin{align*}
        &\Pr\left(v \in S \mid H(G) = H\right) = \Pr\left(|\Pa(v)| = 1 \mid H(G) = H\right)\\ &= \frac{h_1 \cdot p (1-p)^{h_1-1}}{\sum_{i = 1}^s\binom{h_1}{s}p^s(1-p)^{h_1-s}}
        =\frac{h_1 \cdot p (1-p)^{h_1-1}}{1 - (1-p)^{h_1}}\\ &\geq  \frac{h_1 \cdot p (1-p)^{h_1-1}}{h_1 \cdot p} \geq (1-p)^{h_1}.
    \end{align*}
    Here, we used Bernoulli's inequality in the third line of the derivation. This implies $\E[|S(G)| \mid H(G) = H] \geq (1-p)^{h_1}\cdot |\widetilde{H_2}|$. Since the events $v \in S(G)$ are again independent for different $v \in \widetilde{H_2}$ conditional on the tower distribution, we get by the Chernoff bound (Lemma~\ref{lem:chernoff}):
    \begin{align}
        \notag &\Pr\left(|S(G)| \leq \frac{(1-p)^{h_1}}{2} \cdot |\widetilde{H_2}| \middle | H(G) = H\right)\\ &\leq e^{-(1-p)^{h_1}|\widetilde{H_2}|/8}. \label{eq:conc-S}
    \end{align}
    Now, let $\cH$ be the set of tower decompositions $H = (H_1, \dots, H_s)$ satisfying $|H_1| \leq 5p^{-1}$ and $|H_2| \geq p^{-1}/\log^2(p^{-1})$. For $H \in \cH$, we get
    \begin{equation*}
        (1-p)^{h_1} \cdot |\widetilde{H_2}| \geq (1-p)^{5/p} \cdot \frac{p^{-1}}{\log^2(p^{-1})} \geq \frac{e^{-5}}{2}\cdot \frac{p^{-1}}{\log^2(p^{-1})},
    \end{equation*}
    where the last steps holds when $p$ is small enough. Set $\alpha(p) := p^{-1} / (4e^5 \log^2(p^{-1}))$. Then, for $H \in \cH$, equation~\eqref{eq:conc-S} implies 
    \begin{align*}
        &\Pr\left(|S(G)| \leq \alpha(p) \mid H(G) = H\right)\ \leq e^{-\alpha(p)/2}.
    \end{align*}
    Summing over all $H \in \cH$ gives 
    \begin{align*}
        \Pr\left(|S(G)| \leq \alpha(p) \wedge H(G) \in \cH\right) \leq e^{-\alpha(p)/2}.
    \end{align*}
    And by Lemma~\ref{lem:lower-tailbd} and~\ref{lem:upper-tailbd}, we get 
    \begin{align*}
        \Pr\left(|S(G)| \leq \alpha(p)\right) &\leq e^{-\alpha(p)/2} + \Pr\left(H(G) \notin \cH\right)\\ &= o(1). 
    \end{align*}
    Now, summing equation~\eqref{eq:conc-X} over all graphs with $S(G) \geq \alpha(p)$ gives 
    \begin{align*}
        \Pr\left(|X(G)| \leq \alpha(p)/4 \wedge S(G) \geq \alpha(p)\right) \leq e^{-\alpha(p)/16}.
    \end{align*}
    So, we finally obtain
    \begin{align*}
        \Pr\left(|X(G)| \leq \alpha(p)/4\right) &\leq e^{-\alpha(p)/16} + \Pr\left(S(G) \leq \alpha(p)\right)\\ &= o(1). 
    \end{align*}
    Hence, with high probability, $G$ has a matching of reversible layer-2-edges of size at least $\alpha(p)/4$.
\end{proof}

\section{Proofs for Directed Cyclic Graphs} \label{apx:dcgs}

The following statement is a simple observation of the properties of the construction given in Definition~\ref{defn:reverse}.

\begin{lemma}\label{lem:reverse-properties}
Let $G$ be a directed cyclic graph containing a cycle $C$. Let $H = \rgc$ and let $\overline{C}$ be the reversed version of $C$ that occurs in $H$. We have the following properties: 
\begin{enumerate}
    \item each vertex has the same descendants in $G$ as in $H$;
    \item $G = \textsc{Reverse}(H, \overline{C})$.
\end{enumerate}
\end{lemma}

Now, we prove the key property of the construction in Definition~\ref{defn:reverse}.

{
\renewcommand{\thetheorem}{\ref{prop:ME-reverse}}
\begin{prop}
    Let $G$ be a directed cyclic graph containing a cycle $C$. Then, the graph $H = \textsc{Reverse}(G, C)$ is Markov equivalent to $G$. 
\end{prop}
\addtocounter{theorem}{-1}
}

\begin{proofof}{Proposition~\ref{prop:ME-reverse}}
    Suppose $C = (v_1, \dots, v_k)$ in $G$ and $H = \textsc{Reverse}(G, C)$. Fix an arbitrary conditioning set $Z \subseteq V(G)$. Let $\pi = (w_1, \dots, w_t)$ be a simple active path between $w_1$ and $w_t$ given $Z$ in $G$. The first part of the proof is to show that there exists an active path from $w_1$ to $w_t$ given $Z$ in $H$. \\

    \textbf{Case 1:} Suppose some vertex of $C$ is in $Z$, i.e. $C \cap Z \neq \emptyset$. Consider the path $\pi'$ in $H$ obtained as follows: 
    \begin{itemize}
        \item whenever $w \to v_i$ occurs in $\pi$ for $v_i \in C$, replace it with $w \to v_{i-1} \leftarrow v_i$ (similarly, if $v_i \leftarrow w$ occurs in $\pi$, replace it with $v_i \to v_{i-1} \leftarrow w$). 
        \item make the resulting path simple by deleting any self-loops if they exist.
    \end{itemize}
    By definition of $H = \rgc$, $\pi'$ is a valid simple path from $w_1$ to $w_t$ in $H$. We show that $\pi'$ is also an active path given $Z$. 
    
    First, consider any segment $a - b - c$ in $\pi'$ with $b \notin C$. If $b$ is a collider in the segment, i.e., we have $a \to b \leftarrow c$, then $a \to b \leftarrow c$ must occur in $\pi$ as well, since incoming edges to $b \notin C$ are never changed between $\pi$ and $\pi'$. This means $b$ has a $G$-descendant in $Z$, and by part 1 of Lemma~\ref{lem:reverse-properties}, $b$ must also have an $H$-descendant in $Z$, so $a \to b \leftarrow c$ is active in $\pi'$. If $b$ is a non-collider in $a - b - c$, then it must also occur as a non-collider in $\pi$ (but perhaps in a different segment). This is because outgoing edges from $b$ in $\pi'$ that do not exist in $\pi$ must have been a replacement for another outgoing edge from $b$ in $\pi$. Since $\pi$ is active, this means $b \notin Z$, and hence the segment $a - b - c$ is active in $\pi'$. 

    Since $C$ contains a vertex of $Z$, any collider in $C$ is automatically active. To show that the entire path $\pi'$ is active, it now suffices to show that any vertex $v_i \in C$ that is a non-collider in $\pi'$ must have been a non-collider in $\pi$ as well. Indeed, if $v_i$ is a non-collider in $\pi'$ with an outgoing edge $v_i \to w$ with $w \neq v_{i-1}$, then this outgoing edge must exist in $\pi$ or must be a replacement for another outgoing edge in $\pi$ from $v_i$, so $v_i$ is also a non-collider in $\pi$. If $v_i$ is a non-collider in a circle segment $v_{i+1} \to v_i \to v_{i-1}$ in $\pi'$, then $v_i$ must occur as a non-collider in the segment $v_{i+1} \leftarrow v_i \leftarrow v_{i-1}$ in $\pi$. Finally, we could have $w \to v_i \to v_{i-1}$ in $\pi'$ for $w \neq v_{i+1}$, but this is only possible if $w \to v_{i+1} \leftarrow v_i \leftarrow v_{i-1}$ occurred in $\pi$. In each case, $v_i$ is also a non-collider in $\pi$, which concludes the argument. 
    
    \textbf{Case 2:} Suppose none of the vertices in $C$ is part of the conditioning set $Z$, i.e. $C \cap Z = \emptyset$. In this case, we have to construct the path $\pi'$ in $H$ in a different way: 
    \begin{itemize}
        \item whenever $w \to v_i$ occurs in $\pi$ for $v_i \in C$, replace it with $w \to v_{i-1} \to v_{i-2} \to \dots \to v_{i}$ (similarly, if $v_i \leftarrow w$ occurs in $\pi$, replace it with $v_i \leftarrow v_{i+1} \leftarrow \dots \leftarrow v_{i-1} \leftarrow w$). 
        \item make the resulting path simple by deleting any self-loops if they exist.
    \end{itemize}
    By definition of $H = \rgc$, $\pi'$ is a valid simple path from $w_1$ to $w_t$ in $H$. We show that $\pi'$ is also an active path given $Z$. 

    By the same argument as in the first case, any segment $a - b - c$ of $\pi'$ with $b \notin C$ must be active. Since $C$ does not contain any vertex in $Z$, any non-collider of $\pi'$ in $C$ is active too. Hence, it suffices to show that every collider of $\pi'$ in $C$ is active. If $\pi'$ contains a segment $w \to v_i \leftarrow u$ with $w, u \neq v_{i+1}$, then $\pi$ must contain the segment $w \to v_{i+1} \leftarrow u$. Since $\pi$ is active, $v_{i+1}$ must have a $G$-descendant in $Z$, and by part 1 of Lemma~\ref{lem:reverse-properties}, this is also an $H$-descendant. Hence, also $v_i$ has an $H$-descendant in $Z$, so $w \to v_i \leftarrow u$ is active in $\pi'$. Finally, note that $\pi'$ cannot contain a segment of the form $w \to v_i \leftarrow v_{i+1}$, by construction. This completes the first part of the proof. 

    We have shown that, whenever $a$ and $b$ are d-separated in $H = \rgc$ given $Z$, they also must be d-separated in $G$ given $Z$. However, since $G = \textsc{Reverse}(H, \overline{C})$ by part 2 of Lemma~\ref{lem:reverse-properties}, the converse must also hold. Hence, $G$ and $H$ are Markov equivalent. 
\end{proofof}

\end{document}